\documentclass[runningheads,a4paper]{llncs}

\usepackage{graphicx}
\usepackage{latexsym}
\usepackage{amssymb,amsmath,amscd}
\usepackage{graphicx}

\usepackage{url}

\DeclareMathOperator{\divergence}{div}
\DeclareMathOperator{\grad}{grad}
\newcommand{\R}{\mathbb R}
\newtheorem{thm}{Theorem}

\newtheorem{prop}[thm]{Proposition}

\newcommand{\keywords}[1]{\par\addvspace\baselineskip
\noindent\keywordname\enspace\ignorespaces#1}

\begin{document}

\mainmatter  

\title{Schr\"{o}dinger Diffusion for Shape Analysis with Texture}

\titlerunning{Schr\"{o}dinger Diffusion for Shape Analysis with Texture}

%
%
\author{Jose A. Iglesias \and Ron Kimmel}
\authorrunning{Jose A. Iglesias \and Ron Kimmel}

\institute{Department of Computer Science, Technion-Israel Institute of Technology,\\
Taub Building, Technion, Haifa 32000, Israel}

%
%

\toctitle{Lecture Notes in Computer Science}
\maketitle

\begin{abstract}
In recent years, quantities derived from the heat equation have become popular in shape processing and analysis of triangulated surfaces. Such measures are often robust with respect to different kinds of perturbations, including near-isometries, topological noise and partialities. Here, we propose to exploit the semigroup of a Schr\"{o}dinger operator in order to deal with texture data, while maintaining the desirable properties of the heat kernel. We define a family of Schr\"{o}dinger diffusion distances analogous to the ones associated to the heat kernels, and show that they are continuous under perturbations of the data. As an application, we introduce a method for retrieval of textured shapes through comparison of Schr\"{o}dinger diffusion distance histograms with the earth's mover distance, and present some numerical experiments showing superior performance compared to an analogous method that ignores the texture.
\end{abstract}

\section{Introduction}

There is an ever growing quantity of 3D shapes available, either scanned from real objects, manually modelled by artists, or acquired from other sources. Adequately classifying them, and being able to find similar and dissimilar models is therefore increasingly important, and automatic solutions are needed for the goal of efficient computerized \textit{shape retrieval}.

For a retrieval method to be useful, given the variability of shapes, often some invariance properties are required. The most obvious one is translation and rotation (that is, Euclidean) invariance. Their scale is often arbitrary, so sometimes it is also interesting to enforce invariance with respect to global or local scaling. More challenging is recognition in classes of non-rigid shapes, like shapes representing human faces, animals or animated characters. In these cases, only the intrinsic geometry can be used, thus enforcing invariance to isometries, and robustness with respect to near-isometries.

Often, geometric models include textures, which are an integral part of the representation. It is therefore natural to try and use the texture information to better distinguish between objects, for example in cases like separating between a horse and a zebra, classifying different species of fish with similar shapes but different colors and patterns, or categorizing archaeological findings. 

In this paper, we introduce a representation that incorporates texture data within several recent methods of shape analysis and retrieval, which themselves depend only on intrinsic geometry and hence are appropriate for non-rigid shapes. Our method fully inherits the desirable invariance and robustness properties of these methods, while also utilizing the texture of the shapes.

The paper is organized as follows. In Section 2 we briefly review previous efforts and basic concepts on which our method is based. In Section 3, we define our central quantities, a family of diffusion distances based on diffusion with Schr\"{o}dinger operators incorporating the texture data, and present some theoretical results about them. Then, in Section 4, we present a system of shape retrieval based on comparison of histograms of Schr\"{o}dinger diffusion distances with the earth mover's distance. Finally, in Section 5, we present some experimental results obtained with our representation model.

\section{Diffusion in Shape Analysis. Previous Works}

Adding to a long history of use of Laplace operators in geometry processing applications \cite{lap-mesh-pro}, the spectral decomposition of Laplace-Beltrami operators on surfaces has proven to be useful for tasks of shape analysis and comparison \cite{gps}.

In \cite{sog-hks}, diffusion through the heat equation, constructed from the spectral decomposition, was used for comparison of shapes through the introduction of the heat kernel signature (HKS). Since then, many methods have used descriptors for shapes built from the heat kernel \cite{shape-google}\cite{sihks}.

Diffusion distances were introduced by Coifman and Lafon in \cite{coifman06} for data analysis of point clouds, under the basic assumption that the sampled points come from an underlying low-dimensional manifold. Recently, diffusion distances have also received considerable attention for shape analysis and retrieval tasks, for example in shape recognition \cite{shapespec} or shape matching \cite{gh-diff}.

Recently, an approach to shape retrieval including texture data was proposed in \cite{artiom}, introducing three channels of texture (in the Lab color space) through a higher dimensional embedding, similar to the Beltrami framework \cite{sochkim}. In comparison, the method presented here supports a single channel for the texture, yet it has a clear interpretation in terms of diffusion on the original shape. It also requires lower order derivatives of the texture (at most one in our case, versus two for the embedding approach), making it less sensitive to noise.

\section{Schr\"{o}dinger Operators and Diffusion}

We will consider our surfaces to be compact two-dimensional manifolds, denoted by $M$ and embedded in $\R^3$, with triangular meshes as discretizations. In terms of a local parametrization and the corresponding first fundamental form $g$, one can define the Laplace-Beltrami operator through the formula
\begin{equation}\Delta_g f=(\divergence_g \circ \grad_g)(f)=\sum_{i,j}\frac{1}{\sqrt{|g|}}\partial_i\left(\sqrt{|g|} g^{ij}\partial_j f\right),\end{equation}
where $|g|$ is the determinant of the metric, $g^{ij}$ are the components of the inverse of the metric, and $\partial_i$ denotes partial derivative with respect to the $i$-th coordinate. It is well known \cite{jost2008riemannian} that the Laplace-Beltrami operator doesn't depend on the coordinate functions chosen, and since it's defined in terms of $g$, it is invariant under transformations that preserve $g$, that is, isometries. This operator is a generalization of the standard Laplacian in $\R^n$, for many of the processes associated to the former, like diffusion and smoothing.

Consider a function $V:M\rightarrow \R$, which we require to be bounded, but without needing any further regularity, in particular not necessarily differentiable or even continuous, to which we will refer as the \textit{potential}. A Sch\"{o}dinger operator on the surface $M$ is an operator of the form $\Delta_g-V$, with $V$ being considered as a multiplication operator, that is, $Vf(x)=V(x)f(x)$. 

These operators are referred to as Schr\"{o}dinger operators, and play a major role in quantum mechanics, where the study of their spectrum is key to understanding the Schr\"{o}dinger equation. In those cases the potential is usually unbounded, which makes the analysis a challenge of its own. One can also consider the diffusion equation associated to these operators,
\begin{equation}\label{sdiff}\begin{cases}
 \partial_t u(x,t)=\Delta_g u(x,t)-V(x)u(x,t)\\
  u(x,0)=u_0(x),\end{cases} \end{equation}
which will be the equation defining the quantities that we will use in what follows. Boundary conditions are not needed since $M$ is compact. The following result asserts that for these operators, on the continuous level, everything works as expected, mimicking the situation with the Laplacian and heat kernels:
\begin{thm}\label{webgood}Let $(M,g)$ be a compact Riemannian manifold of class $C^2$, $\Delta_g$ the Laplace-Beltrami operator on $(M,g)$, and $V\in L^{\infty}(M,\mu_g)$, where $\mu_g$ is the measure associated to the Riemannian volume element, with $V\geq0$. Then, the operator $-\Delta_g+V$ admits a spectral decomposition $\{(\phi_j, \lambda_j)\}_{j=1}^{\infty}$, such that $\{\phi_j\}_{j=1}^{\infty}$ is an orthonormal basis for $L^2(M,\mu_g)$, $\lambda_j \geq 0$ and $\underset{j\rightarrow \infty}{\lim} \lambda_j=+\infty$.

Moreover, there exists a family of functions $h_t \in L^2(M,\mu_g)$, such that for all $u_0 \in L^2(M,\mu_g)$, the unique solution of (\ref{sdiff}) is given by
\begin{equation}u(x,t)=\int_S h_t(x,y)u_0(y)d\mu_g(y),\end{equation}
and the following formula holds
\begin{equation}\label{krnl}h_t(x,y)=\sum_{i=1}^{\infty}e^{-\lambda_i t}\phi_i(x)\phi_i(y).\end{equation}
\end{thm}
\begin{proof}The proof is essentially the same as in the Laplacian case. We refer to \cite{grigoryan} for a standard proof. Let us note, however, that both the hypotheses that the manifold is compact and the potential $V$ is bounded are essential, as otherwise, the discreteness of the spectrum is not guaranteed, since in those cases the involved resolvents could fail to be compact.
\end{proof}

Based on this result, a squared diffusion distance \cite{coifman06} can be defined as the $L^2$ norm of the difference of the kernels for Equation (\ref{sdiff}), and using Equation (\ref{krnl})
\begin{equation}d_t^2(x,y)=\left\|h_t(x,\cdot)-h_t(y,\cdot)\right\|^2_{L^2}=\sum_{j=1}^{\infty} e^{-2\lambda_j t} (\phi_j(x)-\phi_j(y))^2.\end{equation}
Note, that the sign in the exponential arises because we have defined $\lambda_j>0$ to be the eigenvalues of $-(\Delta_g-V)$. These distances enjoy the same properties as those associated to the heat kernel, since all the properties proved in \cite{coifman06}, including the fact that the formula above defines a distance, are valid for more general semigroups and not just the one associated to the Laplacian.

We also have that the solutions to the diffusion equation, and therefore any quantities derived from it, in particular our diffusion distances, are continuous with respect to perturbations of the potential. Namely,

\begin{prop}\label{stability}Consider the problem:
\begin{equation}\begin{cases}
 \partial_t u(x,t)=\Delta u(x,t)-(V(x)+\epsilon N(x))u(x,t)\\
  u(x,0)=u_0(x),\end{cases} \end{equation}
where $V,N\in L^{\infty}(M,\mu_g)$, $V\geq 0$ and $V+\epsilon N \geq 0$ for some $\epsilon \geq 0$. Then, the solutions to of the above problem converge strongly in $L^2(M,\mu_g)$ to the ones of problem \ref{sdiff} as $\epsilon \rightarrow 0$, for each fixed $t>0$.\end{prop}

Proposition \ref{stability} can be proved by using standard results in perturbation theory of linear semigroups. We provide a proof in appendix A.

\subsection{The Feynman-Kac Formula}

To shed some light on the behavior of the solutions to Equation (\ref{sdiff}), we provide an informal discussion on a well-known stochastic interpretation of such solutions.

The Feynman-Kac formula \cite{simon2005functional} expresses the solution of a diffusion equation in terms of Brownian motion (strictly, the Wiener process on our space, $X$)
\begin{equation}\label{f-k}u(x,t)=E\left(u_0(X_0)\exp\left(-\int_0^t V(X_\tau) d\tau\right)\mid X_t=x\right),\end{equation}
the conditional expectation meaning that we take averages of all the Wiener paths that reach $x$ at time $t$, starting from elsewhere. Note that the integral inside the exponential involves the Wiener process itself, and hence needs to be understood in the sense of stochastic integrals. A rigorous treatment of this is beyond the scope of this paper, but can be found in \cite{simon2005functional}.

In the case of the heat equation, $V=0$, the initial values are transported over random paths, and the expected value over all paths that reach a point at a given time is the value of our solution. This kind of averaging property is the reason behind the robustness to different kinds of noise that diffusion distances and other quantities derived from heat kernels possess.

For Schr\"{o}dinger operators, one can think of this transported value being modulated exponentially by the potential $V(x)\geq0$, in a way consistent with what one gets by disregarding the diffusion term to end up with $u_t+Vu=0$. The transported values will be decreased according to how large $V$ is in the areas that the Brownian motion crosses, on average. Figure \ref{fig:pot_barr} illustrates this behavior.

\begin{figure}[ht]
\begin{center}\includegraphics[width=1.0\linewidth]{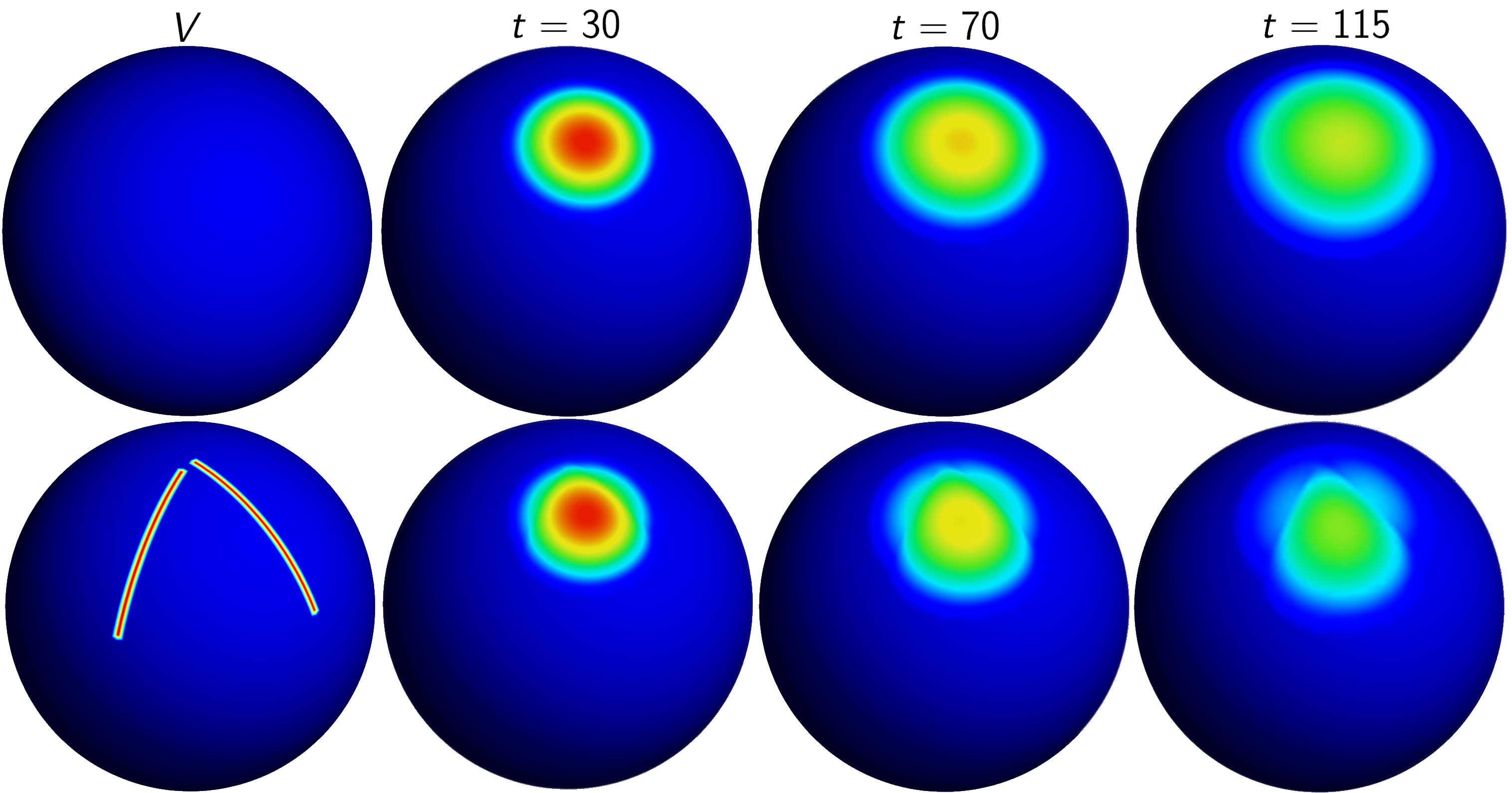}\end{center}
\caption{From left to right, row-wise: Potential and kernel $h_t(p,\cdot)$ at three different times.}\label{fig:pot_barr}
\end{figure}

\section{Textured Shape Retrieval with Schr\"{o}dinger Diffusion}

In our context, all references to textures on shapes will in fact be about vertex colorings. This is because only at that level the notions of a mapping on the surface and the diffusion can have discretizations consistent with one another, as required for our Schr\"{o}dinger operators. In the case of textures mapped on triangles, one could induce a vertex map by averaging over the Voronoi region corresponding to the vertex in its one-ring neighborhood, for example.

For our shape retrieval application, we intend to define a distance between signatures of the shapes, which in our case will be histograms of Schr\"{o}dinger diffusion distances between points. The retrieval method would then select from a database the shapes with the smallest distance to the query. In what follows, we assume our texture is a differentiable function $I:M\rightarrow \R$ in the continuous model, and a vertex function $I(v_i)$ for the discretized version.

\subsection{Operator Discretization}
A popular discretization for the Laplace-Beltrami operator of a surface is the so-called cotangent weight scheme \cite{pinkall-polthier}, where the weights are the sums of cotangents of angles adjacent to the edge, and the normalization coefficients are the Voronoi areas corresponding to the vertex. One of its disadvantages is the fact that for meshes with obtuse angles, the edge weights become negative, making it unsuitable \cite{nofreelunch} for simulating diffusion processes. 

So instead, we chose the `Mesh Laplacian' discretization of \cite{meshlap}, inspired by the one introduced by Belkin and Niyogi for data analysis of point clouds in \cite{lap-eigen}. For triangular meshes, it is given by
\begin{equation}\Delta_s f(v_i)=\frac{1}{4 \pi s^2}\sum_{\tau \in \mathcal{T}(v_i)}\frac{A(\tau)}{3}\sum_{w \in \tau}e^{\frac{-\|v_i-w\|^2}{4s}}(f(w)-f(v_i)),\end{equation}
where $\mathcal{T}(v_i)$ denotes the set of triangles in a neighborhood of $v_i$, $A(\tau)$ is the area of the triangle $\tau$ and $s$ is a scaling parameter. 

This discretization has several advantages. The ones that are the most useful for us are that the weights are nonnegative by definition, and the fact that it converges pointwise to the continuous Laplace-Beltrami operator, when the meshes approximate a smooth surface and $\mathcal{T}(v_i)$ is always the whole shape, as proved in \cite{meshlap}. Spectral convergence in a probabilistic sense, of the point-cloud version of the operator, is proved in \cite{belkConv}. In our case, since it would not be practical to use the whole shape as a neighborhood, we have taken the neighborhoods $\mathcal{T}(v_i)$ to be the one-ring neighborhood of each vertex. The parameter $s$ was choosen in a uniform way, not taking into account the size of the different neighborhoods, as one fifth of the median of the edge lengths over the whole shape.

After discretizing the Laplace-Beltrami operator as above, our discrete operators are defined in the obvious way,
\begin{equation}((\Delta_s -V)f)(v_i)=\Delta_s f(v_i)-V(v_i)f(v_i).\end{equation}

\subsection{Choice of Potential and its Discretization}

One choice for the potential $V$ would be just to take $V=I$, $I$ corresponding for example to the luminance of the texture. This would have the advantage of not having to explicitly compute any derivative of $I$, but in turn would make it depend on the reference taken for the texture, that is, on transformations of the kind $\tilde{I}=I+c$, where $c$ is a constant.

Another option is to use an edge indicator for the textures as the potential $V$, the most straightforward being the modulus of the surface gradient, $V=|\nabla I|$, $\nabla$ being the (Riemannian) gradient on the surface. Other options are $V=\log(1+I)$ or $V=\log(1+|\nabla I|)$, as a way to mitigate the exponential decay caused by the potential $V$ in Formula (\ref{f-k}). In our experiments below, we used this last potential.

Intuitively, from the random walk interpretation and the Feynman-Kac Formula (\ref{f-k}), we see that it will be harder to diffuse across edges of the texture, while in constant areas, the behavior will be that of the usual heat equation. This can be appreciated in Figure \ref{fig:diffexmp}.

We will now describe the approximation of the gradient of a function on the surface that we employed, which is the same as in \cite{sethian-vladimirsky}. Let $v$ be the vertex we are interested in, and $\tau=\{u,v,w\}$ a triangle of the shape having $v$ as a vertex. Denote
\begin{equation}r_u=\frac{u-v}{|u-v|},\:r_w=\frac{w-v}{|w-v|},\:P=\begin{bmatrix} r_u^T \\ r_w^T \end{bmatrix},\end{equation}
where $r_u^T, r_w^T$ denote transposes. Note, that $P$ is the change of basis matrix from the canonical basis to $r_u, r_w$, which form a basis for the plane that contains the triangle. Also, denote by $I_u,I_v,I_w$ the values of our function on the three vertices. Then, we can consistently approximate the norm of the gradient by
\begin{equation}\left(|\nabla I|_v^\tau\right)^2=\begin{pmatrix}\frac{I_u-I_v}{|u-v|} & \frac{I_w-I_v}{|w-v|}\end{pmatrix} (PP^T)^{-1} \begin{pmatrix}\frac{I_u-I_v}{|u-v|} \\ \frac{I_w-I_v}{|w-v|}\end{pmatrix}.\end{equation}
Finally, to obtain our discretization, we average this approximation over the one-ring neighborhood of $v$,
\begin{equation}|\nabla I|_v=\frac{1}{\#\mathcal{T}(v)}\sum_{\tau \in \mathcal{T}(v)}|\nabla I|_v^\tau.\end{equation}

\begin{figure}[t]
\begin{center}\includegraphics[width=0.9\linewidth]{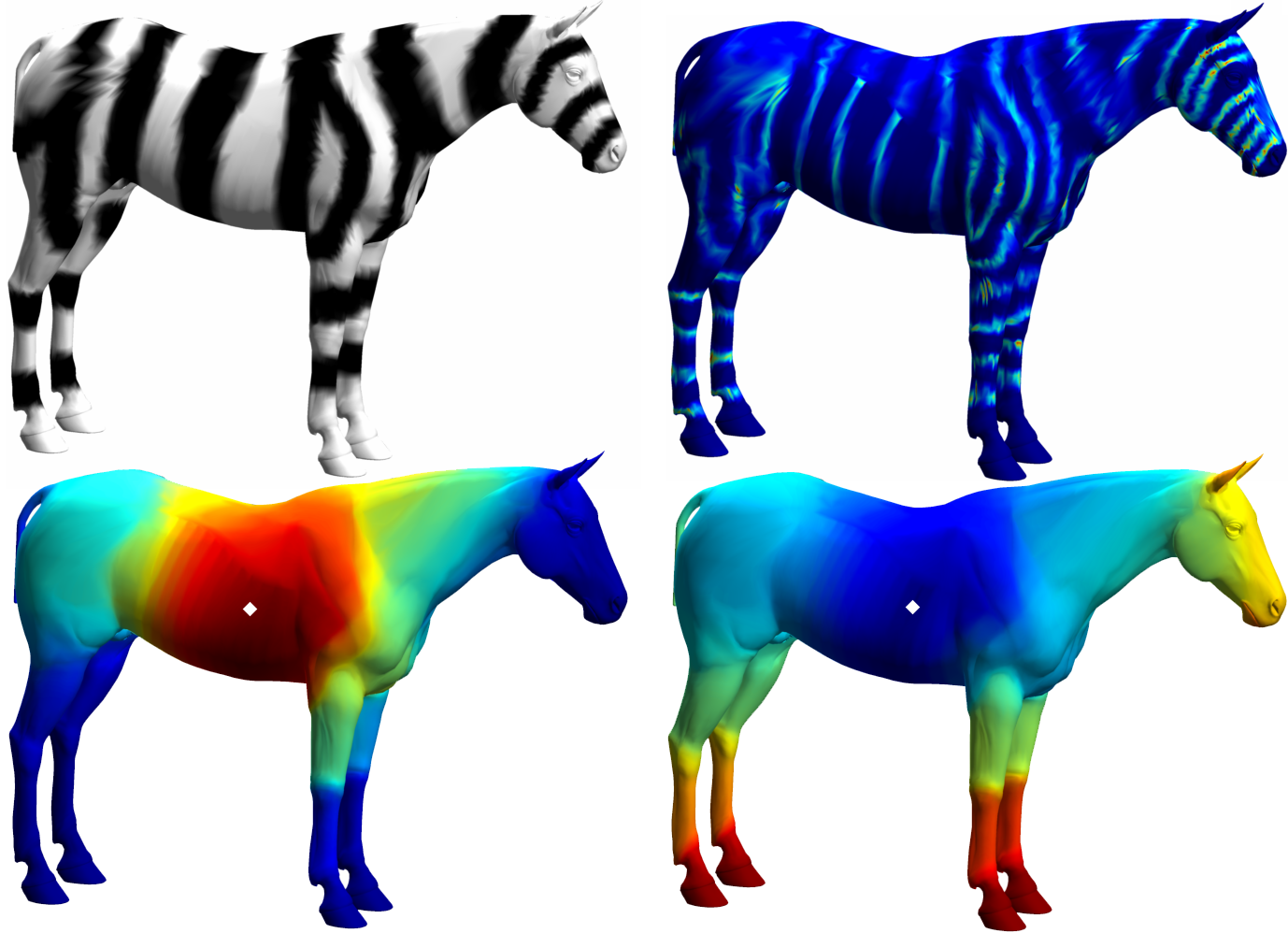}\end{center}
\caption{Diffusion on a textured shape, with $V=\alpha(1+\beta|\nabla I|)$. Left to right, row wise: textured shape, modulus of the gradient, snapshot of kernel, diffusion distance. Observe that some of the stripes of the zebra can be seen in the distance and the kernel is shaped by them, since texture edges work against the diffusion. The source is marked in white.}
\label{fig:diffexmp}
\end{figure}

\subsection{Point Selection and histogram comparison}

To obtain a signature based on diffusion distances, one needs to select points between which the distances are computed. In our case, we performed farthest point sampling \cite{hochbaum85} based on Euclidean ambient distances (for simplicity) to pick $100$ points, for which we computed the diffusion distance map to the rest of the shape. Then, the results for each of them are combined in a global histogram, which will serve as descriptor for the shape. The histograms were quantized with $120$ bins, and normalized, to compare shapes with any number of vertices.


After obtaining the histograms of diffusion distances to be used as signatures, we need a way to compare them consistently. A popular approach for comparing probability distributions (and hence normalized histograms) is the Earth Mover's Distance, or EMD \cite{emd}. We use it to compute distances between the histograms of diffusion distances, as a means of comparing them. Let us note that even if in the one-dimensional continuous case computing the optimal cost is straightforward from the cumulative distribution functions, in the discrete case one still needs to solve a flow network optimization problem, which is computationally expensive, if done naively \cite{emd-l1}. In our case, we have used the method and code of Pele and Werman \cite{pele-emd} to efficiently approximate the EMD between our histograms.

A similar approach to shape retrieval, but with local descriptors, was used in \cite{inner-distance}, where the distances used were the inner distances inside planar shapes. A more sophisticated approach using histograms of geodesic distances and Wasserstein metrics can be found in \cite{shapes-ot}. We emphasize that we purposefully used a classification method which is far from being state of the art, to be able to better demonstrate the increased performance with just a moderate-sized database.

\section{Experimental Results}

For our experiments, the shapes were taken from the TOSCA nonrigid shape database \cite{bronstein2008numerical}, and the textures were manually added as vertex colorings. The final database consisted of $73$ shapes, belonging to $8$ different classes. Inside each class, the shapes differed by an almost-isometric deformation (different `poses').

The Laplace-Beltrami operators were discretized through the scheme described above, before combining them with the norms of the gradient of the texture data, as described in Section 4. Diffusion distances were computed directly from the definition, using $100$ eigenvalues and eigenfunctions, which were in turn computed from the corresponding matrices of the discrete operators.

Then, the EMD between the histograms from each pair of shapes was computed, and finally the two closest matches for each shape were chosen as candidates for retrieval. This whole process was done for $V=0$ and $V=\alpha\log(1+\beta|\nabla I|)$, with $\alpha$ and $\beta$ normalization constants which control the resistance to diffusion induced by the texture. 

Results for the distances are shown in Figure \ref{fig:distmatrix}, and Table \ref{table:results} shows the amount of correct matches and averages of distances inside and outside the classes, after normalization with respect to the maximum distance between elements.

As we can see, the average distance between shapes in the same class is reduced, while the average distance between shapes in different classes increases, providing better separation between the classes. Also, some queries that would produce incorrect matches with just the geometry are correctly matched when also using the texture.

\begin{figure}[ht]
\begin{center}\includegraphics[width=1.0\linewidth]{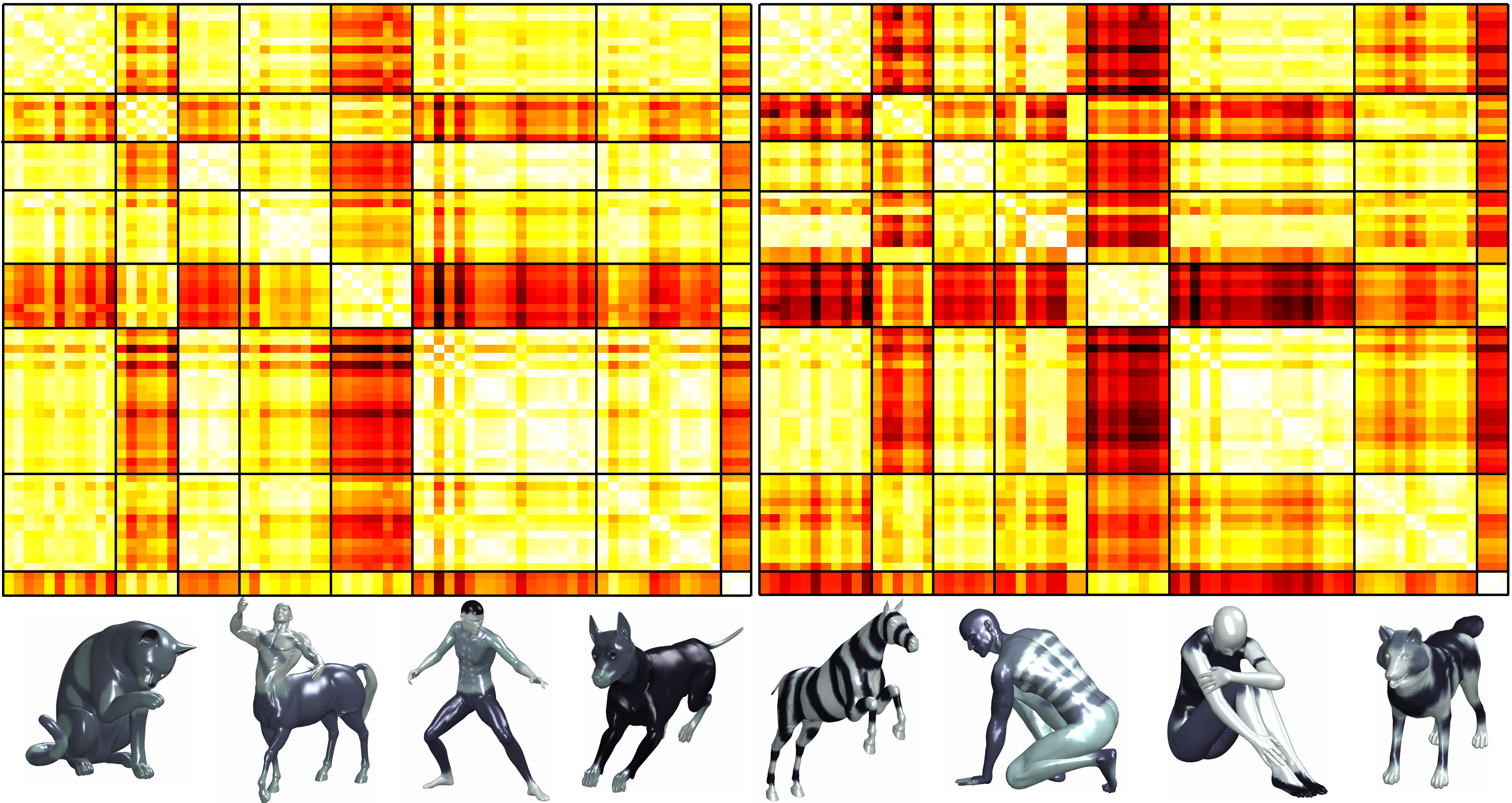}\end{center}
\vspace{-15pt}
\caption{Distance matrices between the signatures of the shapes. Lines indicate separation between classes. White corresponds to zero distance, black to maximum distance. Left: Results without texture. Right: Results with texture. Lower: Representatives of each class of shapes, in the same order as the matrices.}\label{fig:distmatrix}
\vspace{-20pt}
\end{figure}

\begin{table}[htb]
		\begin{center}	
		\caption{Numerical results for both cases}\label{table:results}	
		\vspace{-5pt}
    \begin{tabular}{ | p{7.9cm} | c | c |}
    \hline
     & $V=0$ & $V=\alpha\log(1+\beta|\nabla I|)$ \\ \hline
    Nearest shape belongs to correct class     & 52/73 & 65/73\\ \hline
		Second nearest shape belongs to correct class    & 41/73 & 50/73\\ \hline
		Normalized avg. distance for shapes in same class & 0.1440 & 0.1167\\ \hline
		Normalized avg. distance for shapes in different classes & 0.2905 & 0.3822\\ \hline				    		
    \end{tabular}						
		\vspace{-30pt}
		\end{center}
\end{table}

\section{Conclusions}

We defined a family of diffusion distances based on the diffusion associated to Schr\"{o}dinger operators, for use in triangulated shapes with textures, where quantities derived from the texture are introduced as the potential part of the operator. These are at least continuous with respect to perturbations of the texture. 

The practical usefulness of our method was illustrated by a simple retrieval example using global diffusion distance histograms as descriptors. Using the available texture information resulted in better performance than using only the geometry data through standard heat diffusion. 

Our approach could also be useful to incorporate texture data in other methods of shape analysis using Laplacian operators but not heat diffusion explicitly, such as the ones in \cite{wks} and \cite{gps}.

\subsubsection*{Acknowledgments.}
This work was supported by the European Commission ITN-FIRST, agreement No. PITN-GA-2009-238702.

\bibliographystyle{splncs}
\bibliography{schrodinger-cameraready}

\appendix
\section{Proof of Proposition 2}

We want to apply \cite{kato} IX.2.16, which tells us that for strong convergence of a semigroup $e^{tT_\epsilon}\rightarrow e^{tT}$ as $\epsilon \rightarrow 0$, it is enough that for some $\xi>0$ larger than the exponential growth of the semigroup, the corresponding resolvents $(T_\epsilon- \xi I)^{-1}$ converge strongly. In this case, since our operators $T=\Delta_g-V$ ($V\in L^{\infty}(M,\mu_g)$) are m-accretive (the semigroup is contractive), any $\xi>0$ will suffice. Choose $f \in L^2(M, \mu_g)$ with $\|f\|_{L^2}=1$, $\xi > \|V\|_{L^{\infty}} + \|N\|_{L^{\infty}}$ such that $\xi$ is not an eigenvalue of $-\Delta_g+V$, and consider the equations
\begin{equation}\begin{split}\label{epsprob}-\Delta_g u_\epsilon + (V+\epsilon N)u_\epsilon + \xi u_\epsilon = f\\
                            -\Delta_g w  + Vw + \xi w = f.\end{split}\end{equation}
Note, that by the Fredholm alternative (\cite{RR}, 8.95)	and the fact that the spectrum of $-\Delta_g+(V+\epsilon N)$ is nonnegative, there is a unique solution for both equations. We can also assume $\epsilon<1$. Multiplying the first equation of (\ref{epsprob}) by $u_\epsilon$ and integrating, we obtain
\begin{align*}\int_M fu_\epsilon d\mu_g&=-\int_M u_\epsilon \Delta_g u_\epsilon  d\mu_g +\int_M (\xi+\epsilon N+V)u_\epsilon^2 d\mu_g=\\
&=\int_M \langle \nabla_g u_\epsilon, \nabla_g u_\epsilon\rangle d\mu_g+\int_M (\xi+\epsilon N+V)u_\epsilon^2 d\mu_g,\end{align*}
so that
\begin{equation}\|u_\epsilon\|^2_{L^2} \leq  \frac{\|f\|_{L^2}\|u_\epsilon\|_{L^2}}{\xi-\epsilon\|N\|_{L^{\infty}}-\|V\|_{L^{\infty}}}\leq \frac{\|u_\epsilon\|_{L^2}}{\xi-\|N\|_{L^{\infty}}-\|V\|_{L^{\infty}}},\end{equation}
and hence $u_\epsilon$ is bounded in $L^2(M, \mu_g)$ independently of $\epsilon$. Without loss of generality we can consider $\epsilon=1/n$, and by elliptic regularity, we can take a subsequence $n_k$ such that $u_{n_k}$ converges weakly to some $u_0 \in H^2(M)$. Now substract both equations in (\ref{epsprob}), multiply by an arbitrary h in $H^2(M)$ and get, using the symmetry of $\Delta_g$
\begin{equation}\label{lim}
  -\int_M (w-u_\epsilon) \Delta_gh d\mu_g+\int_M (\xi+V)(w-u_\epsilon)h d\mu_g=\epsilon \int_M N u_\epsilon hd\mu_g.
  \end{equation}
Notice that $\int_M N u_\epsilon h d\mu_g \leq \|N\|_{L^{\infty}}\|u_\epsilon\|_{L^{2}} \|h\|_{L^{2}}$, so that the right hand side of (\ref{lim}) tends to zero. Since $h$ was arbitrary and $H^2(M, \mu_g)$ is dense in $L^2(M, \mu_g)$ ($M$ being compact), we end up with (understood in weak $L^2$ sense)
\begin{equation}-\Delta_g(w-u_0)+(\xi+V)(w-u_0)=0.\end{equation}
But since $\xi>0$ and the spectrum of $-\Delta_g+(\xi + V)$ is nonnegative, the unique solution is $w-u_0=0$, hence $w=u_0$. The fact that none of this depends the selection of $f$ implies convergence in the operator norm. The statement about the diffusion distances is now clear from their definition and the fact that the kernels are smooth.

\end{document}